\documentclass[conference]{IEEEtran}
\IEEEoverridecommandlockouts
\usepackage{cite}
\usepackage{amsmath,amssymb,amsfonts}
\usepackage{graphicx}
\usepackage{textcomp}
\usepackage{xcolor}
\def\BibTeX{{\rm B\kern-.05em{\sc i\kern-.025em b}\kern-.08em
		T\kern-.1667em\lower.7ex\hbox{E}\kern-.125emX}}

\usepackage{times}
\usepackage{soul}
\usepackage[hidelinks]{hyperref}
\usepackage[utf8]{inputenc}
\usepackage[small]{caption}
\usepackage{url}            
\usepackage{booktabs}       
\usepackage{amsfonts}       
\usepackage{nicefrac}       
\usepackage{microtype}      
\usepackage{algorithm}  
\usepackage{algpseudocode}  
\usepackage{amsthm}
\usepackage{amsmath}
\usepackage{bbm}
\usepackage{xcolor}
\usepackage{graphicx}
\usepackage{amsmath}
\usepackage{amsthm}
\usepackage{booktabs}
\newtheorem{assumption}{Assumption}

\newtheorem{theorem}{Theorem}[section]

\newtheorem{corollary}{Corollary}[theorem]

\newtheorem{lemma}[theorem]{Lemma}
\newtheorem{remark}{Remark}

\usepackage{tabularx}

\begin{document}

	\title{Federated Nonconvex Sparse Learning}
	
	\author{\IEEEauthorblockN{Qianqian Tong, Guannan Liang,  Tan Zhu, Jinbo Bi }
		\IEEEauthorblockA{Department of Computer Science and Engineering\\
			University of Connecticut\\
			Email: \{qianqian.tong, guannan.liang, tan.zhu, jinbo.bi\}@uconn.edu}
	}
	\maketitle

	\begin{abstract}
		Nonconvex sparse learning plays an  essential role in many areas, such as signal processing and deep network compression. Iterative hard thresholding (IHT) methods are the state-of-the-art for nonconvex sparse learning due to their capability of recovering true support and scalability with large datasets. Theoretical analysis of IHT is currently based on centralized IID data. In realistic large-scale situations, however, data are distributed, hardly IID, and private to local edge computing devices. It is thus necessary to examine the property of IHT in federated settings, which update in parallel on local devices and communicate with a central server only once in a while without sharing local data. 
		
		In this paper, we propose two IHT methods: Federated Hard Thresholding (Fed-HT) and  Federated Iterative Hard Thresholding (FedIter-HT). We prove that both algorithms enjoy a linear convergence rate and have strong guarantees to recover the optimal sparse estimator,  similar to traditional IHT methods, but now with decentralized non-IID data. Empirical results demonstrate that the Fed-HT and FedIter-HT outperform their competitor - a distributed IHT, in terms of decreasing the objective values with lower requirements on communication rounds and bandwidth.
	\end{abstract}
	
	\section{Introduction}
	
	
	Federated learning is a privacy-preserving learning framework for large scale machine learning on edge computing devices, and solves the data-decentralized optimization problem:
	\begin{align}\label{Problem_fl}
		\min_{x\in \mathbb{R}^d}f(x) =  \sum_{i=1}^N p_i f_{i}(x), 
	\end{align}
	where $f_i(x) = E_{z \sim \mathcal{D}_i}[f_i(x, z)]$ is the loss function of the $i^{th}$ client (or device) with weight $p_i \in [0,1)$, $\sum_{i=1}^N p_i =1$, $\mathcal{D}_i$ is the distribution of data located locally on the $i^{th}$ client, and $N$ is the total number of clients. 
	Federated learning enables numerous clients to coordinately train a model parameterized by $x$, while keeping their own data locally, rather than sharing them to the central server. 
	Due to the high communication cost, the mini-batch stochastic gradient descent (SGD) has not been the choice for federated learning. The FedAvg algorithm was proposed in \cite{mcmahan2016communication}, which can significantly reduce the communication cost by running multiple local SGD steps, and had become the de facto federated learning method. Later, the client drift problem was observed for FedAvg \cite{hsu2019measuring,karimireddy2019scaffold,reddi2020adaptive}, and the FedProx algorithm came to exist \cite{li2018federated} in which the individual clients attempt to add a proximal operator to the local subproblem to address the issue of FedAvg.
	
	In federated learning, many clients can work collaboratively without sharing local private data mutually or to a central server \cite{konevcny2016federateda,konevcny2016federated}. The clients can be heterogeneous edge computing devices such as phones, personal computers, network sensors, or other computing resources. During training, every device maintains its own raw data and only shares the updated model parameters to a central server. 
	Comparing with the well-studied distributed learning, the federated learning setting is more practical in real life and has three major differences: 1) the communication between clients and/or a central server can be slow, which requires the new sparse learning algorithms to be communication-efficient;
	2) the distributions of  training data over devices can be non-independent and non-identical (non-IID), i.e., for $i \neq j$, $\mathcal{D}_i$ and $\mathcal{D}_j$ are very different; 3) the devices are presumably unbalanced in the capability of curating data, which means that some clients may have more local data than others.
	
	When sparse learning becomes distributed and uses data collected by the distributed devices, the local datasets can be sensitive to share during the construction of a sparse inference model. For instance, meta-analyses may integrate genomic data from a large number of labs to identify (a sparse set of) genes contributing to the risk of a disease without sharing data across the labs \cite{wahlsten2003different,kavvoura2008methods}. Smartphone-based healthcare systems may need to learn the most important mobile health indicators from a large number of users, but personal health data collected on the phone are private \cite{lee2012smartphone}. Because of the parameter sparsity, communication cost can be less than learning with dense parameters. However, the SGD algorithm, widely used to train deep neural nets, may not be suitable because the stochastic gradients can be dense during the training process. Thus, communication efficiency is still the main challenge to deploy sparse learning. For example, the signal processing community has been hunting for more communication-efficient algorithms, due to the constraints on power and bandwidth of various sensors \cite{qin2018sparse}.  It is necessary and beneficial to examine the sparsity-constrained empirical risk minimization problem with decentralized data as follows:
	\begin{equation}\label{Problem_sl}
		\min_{x\in \mathbb{R}^d}f(x) =  \sum_{i=1}^N p_i f_{i}(x), \;\;\;\text{subject to} \;\; \| x \|_0 \leq \tau,  
	\end{equation}
	where $p_i$ and $f_i$ are defined as in (\ref{Problem_fl}), $\|x\|_0$ denotes the $l_0$-norm of a vector $x$ which computes the number of nonzero entries in $x$, and $\tau$ is the sparsity level pre-specified for $x$. 
	Communication-efficient  algorithms for solving   (\ref{Problem_sl}) can be pivotal and generally useful in decentralized high-dimensional data analyses \cite{donoho2006compressed,tropp2007signal,bahmani2013greedy,jalali2011learning}.
	
	Even without the decentralized-data consideration, finding a solution to   (\ref{Problem_sl}) is already NP-hard because of the non-convexity and non-smoothness of the cardinality constraint \cite{natarajan1995sparse}. Extensive research has been done for nonconvex sparse learning when training data can be centralized. The methods largely
	fall into the regimes of either matching pursuit methods  \cite{mallat1993matching,pati1993orthogonal,needell2009cosamp,foucart2011hard} or iterative hard thresholding (IHT) methods \cite{blumensath2009iterative,jain2014iterative,nguyen2017linear}.
	Even though matching pursuit methods achieve remarkable success in minimizing quadratic loss functions (such as the $l_0$-constrained linear regression problems), they require to find an optimal solution to {\em argmin} $f( x )$ over the identified support after hard thresholding at each iteration, which does not have analytical solutions for an arbitrary loss, and can be time-consuming \cite{bahmani2013greedy}. Hence, iterative gradient-based HT methods have gained significant interests and become popular for nonconvex sparse learning.
	
	Iterative hard thresholding methods include the gradient descent HT (GD-HT) \cite{jain2014iterative}, stochastic gradient descent HT (SGD-HT) \cite{nguyen2017linear}, hybrid stochastic gradient HT (HSG-HT) \cite{zhou2018efficient}, and stochastic variance reduced gradient HT (SVRG-HT) \cite{li2016stochastic} methods.
	These methods update the iterate $x^t$ as follows: $x_{t+1} = \mathcal{ H }_{\tau}(x_t -\gamma_t v_t)$ where $\gamma_t$ is the learning rate, $v_t$ can be the full gradient, stochastic gradient or variance reduced gradient at the $t^{th}$ iteration, and $\mathcal{ H }_{\tau}(x): \mathbb{R}^d \rightarrow \mathbb{R}^d$ denotes the HT operator that preserves the top $\tau$ elements in $x$ and sets other elements to $0$. 
	All these centralized iterative HT algorithms  can be extended to their distributed version - Distributed IHT (see Supplementary for detail), in which the central server aggregates (averages) the local parameter updates from each client and broadcasts the latest model parameter to individual clients, whereas each client updates the parameters based on the distributed local data and sends back to the central server. The central server is also in charge of randomly partitioning the training data and distributing them to different clients.  Existing theoretical analysis of gradient-based IHT methods can also be applied to the Distributed IHT, but not suitable to analyze IHT in the federated learning setting with the above three differences.
	
	Distributed sparse learning algorithm has been proposed by \cite{wang2017efficient}, which tries to solve a relaxed $l_1$-norm regularized problem and thus introduce extra bias to  (\ref{Problem_sl}). Even though the variants of the Distributed IHT, such as \cite{patterson2014distributed} and \cite{chen2020hdiht}, have been proposed, they are communication expensive and suffer from bandwidth limits, since information needs to exchange at each iteration. Asynchronous parallel SVRG-HT in shared memory also has been proposed in \cite{li2016stochastic}, which cannot be applied in our scenario. We hence propose federated HT algorithms, which enjoy lower communication costs. 
	
	\textbf{Our Main Contributions} 
	are summarized as follows. 
	
	(a) We develop two communication-efficient schemes for the federated HT method: the Federated Hard Thresholding (Fed-HT) algorithm, which applies the HT operator $\mathcal{ H }_{\tau}$ only at the central server right before distributing the aggregated parameter to clients; and the Federated Iterative Hard Thresholding (FedIter-HT) algorithm, which applies $\mathcal{ H }_{\tau}$ to both local updates and the central server aggregate. This is the first trial to apply HT algorithms under federated learning settings.
	
	(b) We provide the first set of theoretical results for the federated HT method, particularly of Fed-HT and FedIter-HT, under the condition of non-IID data. We prove that both algorithms enjoy a linear convergence rate and have a strong guarantee for sparsity recovery. 
	
	In particular, Theorems \ref{theorem1} (for the Fed-HT) and \ref{theorem3} (for the FedIter-HT) show that the estimation error between the algorithm iterate $x_T$ and the optimal $x^*$,  is upper bounded as:
	$E\|x_T - x^*\| \leq \theta^{T} \|x_0 - x^*\|^2+ g(x^*),$  where $x_0$ is the initial guess of the solution, the convergence rate factor $\theta$ is related to the algorithm parameter $K$ (the number of SGD steps on each device before communication) and the closeness between the pre-specified sparsity level $\tau$ and the true sparsity $\tau^*$, and $g(x^*)$ determines a statistical bias term that is related not only to $K$ but also to the gradient of $f$ at the {\em sparse} solution $x^*$ and the measurement of the  non-IIDness of the data across the devices. The theoretical results help us examine and compare our proposed algorithms. For instance, higher non-IIDness across clients causes a larger bias for both algorithms. More local iterations may decrease $\theta$ but increase the statistical bias.  The statistical bias induced by the FedIter-HT in Theorem \ref{theorem3} 
	matches the best known upper bound for traditional IHT methods \cite{zhou2018efficient}. Thus, for more concrete formulations of the sparse learning problem, such as sparse linear regression and sparse logistic regression, we also provide statistical analysis of their maximum likelihood estimators (M-estimators) when using the FedIter-HT to solve them. 
	
	(c) Extensive experiments in simulations and on real-life datasets demonstrate the effectiveness of the proposed algorithms over standard distributed learning. The experiments on real-life data also show that the extra noise introduced by decentralized non-IID data may actually help the federated sparse learning converge to a better local optimizer. 
	
	\section{Preliminaries}
	
	\begin{table}[h!]
		\caption{Brief summary of notations in this paper}
		\begin{center}
			\begin{tabular}{ll }
				\hline 
				$N, i $ & the total number, the index of clients/devices\\
				$p_i$ & the weight of each loss function on client $i$\\
				$T, t$ & the total number, the index of communication rounds\\
				$K, k$ & the total number, the index of local iterations\\
				$\nabla f_i( \cdot )$ & the full gradient  \\
				$\nabla f_{I^{(i)}}( \cdot )$ & the stochastic gradient over the minibatch $I^{(i)}$  \\
				$\nabla f_{i,z} ( \cdot )$ &  the stochastic gradient over a training example\\
				& indexed by $z$ on the $i$-th device \\
				$\gamma_t$ & the stepsize/learning rate of local update\\
				$\mathbb{I}(\cdot)$   &  an indicator function \\
				$supp( x )$   &  the support of $x$ or the index set of non-zero elements in $x$\\
				$x^*$ & the optimal solution to  (\ref{Problem_sl})\\
				$x^{(i)}_{t, k}$ & the local parameter vector on device $i$ \\
				&at the $k$-th iteration of the $t$-th round \\
				$\tau$  &   the required sparsity level \\
				$\tau^*$  &   the optimal  sparsity level to   (\ref{Problem_sl}), $\tau^* = \|x^*\|_0$\\
				$\pi_{\mathcal{I}}(x)$  & the projector takes only the elements of $x$ indexed in $\mathcal{I}$\\
				$E[\cdot]$, $E^{(i)}[\cdot]$  &  the expectation over stochasticity  across all clients\\
				& and of client $i$ respectively \\
				\hline
			\end{tabular}
		\end{center}
	\end{table}
	
	We formalize our problem as   (\ref{Problem_sl}), and give notations, assumptions and prepared lemmas used in this paper.
	We denote vectors by lowercase letters, e.g. $x$, the $l_2$-norm and the $l_{\infty}$-norm of a vector by $\|\cdot\|$ and $\|\cdot\|_{\infty}$, respectively.  The model parameters form a vector $x \in \mathbb{R}^d$. Let $O(\cdot)$ represent the asymptotic upper bound, $[N]$ be the integer set $\{1, ..., N\}$.
	The support $\mathcal{I}^{(i)}_{t, k+1} = supp(  x^*) \cup supp( x^{(i)}_{t, k}) \cup supp( x^{(i)}_{t, k+1})$, is associated with the $(k+1)$-th iteration in the $t$-th round on device $i$. For simplicity, we use $\mathcal{I}^{(i)}= \mathcal{I}^{(i)}_{t, k+1}$, $\mathcal{I}= \bigcup_{i=1}^{N}\mathcal{I}^{(i)}_{t, k+1}  $ throughout the paper without ambiguity, and $\widetilde{\mathcal{I}}=supp\left(\mathcal{H}_{2N \tau}\left(\nabla f\left(\boldsymbol{x}^{*}\right)\right)\right) \cup supp\left(\boldsymbol{x}^{*}\right) $. 

	We use the same conditions employed in the theoretical analysis of other IHT methods by assuming that the objective function $f(x)$ satisfies the following conditions:
	
	\begin{assumption}\label{assumption4f}
		We assume that the loss function $f_i(x)$ on each device $i$ 
		\begin{enumerate}
			\item is restricted $\rho_s$-strongly convex at the sparsity level $s$ for a given $s \in \mathbb{N}_+$, i.e., there exists a constant $\rho_s > 0$ such that $\forall x_1, x_2 \in \mathbb{R}^d$ with  $\| x_1 - x_2\|_0 \leq s $,   $i\in [N]$, we have 
			$$f_i( x_1) - f_i( x_2) - \langle\nabla f_i( x_2), x_1 - x_2\rangle \geq \frac{\rho_s}{2}\| x_1 - x_2 \|^2;$$
			\item is restricted $l_s$-strongly smooth at the sparsity level $s$ for a given $s\in \mathbb{N}_+$, i.e., there exists a constant $l_s > 0$ such that  $\forall x_1, x_2 \in \mathbb{R}^d$ with $\| x_1 - x_2\|_0 \leq s $, $i\in [N]$, we have 
			$$f_i( x_1) - f_i( x_2) - \langle\nabla f_i( x_2), x_1 - x_2\rangle \leq \frac{l_s}{2}\| x_1 - x_2 \|^2;$$
			\item has $\sigma_i^2$-bounded stochastic gradient variance, i.e.,  
			$$ E^{(i)}[\| \nabla f_{i,z}( x) - \nabla f_i(x)\|^2] \leq \sigma_i^2.$$
		\end{enumerate}
	\end{assumption}
	
	\begin{remark}
		When $s = d$, the above assumption is no longer restricted to the support at a sparsity level, and $f_i$ is actually $\rho_d$-strongly convex and $l_d$-strongly smooth.
	\end{remark}
	
	Following the same convention in federated learning \cite{li2018federated,karimireddy2019scaffold}, we also assume the dissimilarity between the gradients of the local functions $f_i$ and the global function $f$ is bounded as follows. 
	
	\begin{assumption}
		The functions $f_i(x)$ ($i \in [N]$) are $\mathcal{B}$-locally dissimilar , i.e. there exists a constant $\mathcal{B}>1$, such that 
		$$\sum_{i=1}^N p_i \|\pi_{\mathcal{I}}(\nabla f_{i}(x))\|^2 \leq \mathcal{B}^2 \| \pi_{\mathcal{I}}{\nabla f( x)}\|^2$$  for any $\mathcal{I}$.
	\end{assumption}
	
	From the assumptions mentioned in the main text, we have the following prepared lemmas to get ready for our theorems.
	
	\begin{lemma}(\cite{li2016nonconvex}) 
		\label{lemma:HT}
		For $\tau>\tau^*$ and for any parameter $x \in \mathbb{R}^d$, we have
		\begin{align*}
			\| \mathcal{H}_\tau ( x) - x^* \|^2_2 \leq (1+\alpha) \| x - x^*\|^2_2,
		\end{align*}
		where $\alpha = \frac{2 \sqrt{\tau^*}}{\sqrt{\tau - \tau^*}}$ and $\tau^* = \|x^*\|_0$.
	\end{lemma}

	\begin{lemma} 
		A differentiable convex function $f_i( x ): \mathbb{R}^d \rightarrow \mathbb{R}$ is restricted $l_s$-strongly smooth with parameter s, i.e. there exists a generic constant $L_s > 0$ such that for any $x_1$, $x_2$ with $\| x_1 - x_2\|_0 \leq s $ and 
		$$f_i( x_1) - f_i( x_2) - \langle\nabla f_i( x_2), x_1 - x_2\rangle \leq \frac{L_s}{2}\| x_1 - x_2 \|^2,$$ then we have:
		$$\|\nabla f_i( x_1) - \nabla f_i( x_2)\|^2 \leq 2 l_s ( f_i( x_1) - f_i( x_2) + \langle \nabla f_i( x _2), x_2 - x_1\rangle).$$
		This is also true for the global smoothness parameter $l_d$.
	\end{lemma}

	\section{The Fed-HT Algorithm}
	
	In this section, we first describe our first new federated sparse learning framework via hard thresholding - Fed-HT, and then discuss the convergence rate of the Fed-HT.
	
	A high level summary of Fed-HT is described in Algorithm \ref{alg:FedAvgHT1}. The Fed-HT generates a sequence of $\tau-$sparse vectors $x_1$, $x_2$, $\cdots$, from an initial sparse approximation $x_0$. At the $(t+1)$-th round,  clients receive the global parameter update $x_t$ from the central server, then run $K$ steps of minibatch SGD based on local private data. In each step, the $i^{th}$ client updates $ x^{(i)}_{t , k+1} = argmin_x  f_i(x^{(i)}_{t , k}) +\langle g_{t, k}^{(i)}, x - x^{(i)}_{t, k} \rangle + \frac{1}{2\gamma_t} \|x - x^{(i)}_{t, k}\|^2 $ for $k \in  \{0, ..., K-1\}$; 
	Clients send $x_{t,K}^{(i)}$  for $i \in  [N]$ back to the central server; Then the server averages them to obtain a dense global parameter vector and apply the HT operator to obtain a sparse iterate $x_{ t+1}$. Compared with the commonly used FedAvg, the Fed-HT can largely reduce the communication cost because the central server broadcasts a sparse iterate at each of the $T$ rounds.
	
	\begin{algorithm}[H]
		\captionof{algorithm}{Federated Hard Thresholding (Fed-HT) }
		\label{alg:FedAvgHT1}
		\begin{algorithmic}
			\State {\bfseries Input:} The learning rate $\gamma_t$, the sparsity level $\tau$, and the number of clients $N$. 
			\State {\bfseries Initialize} $x_{0}$ 
			\For {$t = 0$ to $T-1$}
			\For {client $i=1$ to $N$ parallel}
			\State $x_{t,1}^{(i)}= x_{t}$
			\For {$k = 1$ to $K $}
			\State Sample uniformly a batch $ I_{t,k}^{(i)} $
			\State $g_{t,k}^{(i)}= \nabla f_{I_{t,k}^{(i)}}(x_{t,k}^{(i)})$
			\State  $x_{t,k+1}^{(i)} = x_{t,k}^{(i)} - \gamma_t g_{t,k}^{(i)} $ 
			\EndFor
			\EndFor
			\State Exact-Average: $x_{t+1} = \mathcal{H}_\tau( \sum_{i=1}^N p_i x_{t,K}^{(i)})$  
			\EndFor
		\end{algorithmic}
	\end{algorithm}
	
	The following theorem characterizes our theoretical analysis of Fed-HT in terms of its parameter estimation accuracy for sparsity-constrained problems. Although this paper is focused on the cardinality constraint, the theoretical result is applicable to other sparsity constraints, such as a constraint based on matrix rank. Then, we have the main theorem and the detailed proof can be found in Appendix.
	
	\begin{theorem}\label{theorem1}
		Let $x^*$ be  the optimal solution to (\ref{Problem_sl}),  $\tau^*=\|x^*\|_0 $, and suppose $f( x )$ satisfies Assumptions 1 and 2. The condition number $\kappa_d = \frac{l_d}{\rho_d}\geq 1$.
		Let stepsize $\gamma_t = \frac{1}{6l_d}$ and the batch size $b_t =\frac{\Gamma_1}{\omega_1^t}$, $\Gamma_1 \geq \frac{\xi_1 \sum_{i=1}^N p_i \sigma_i^2}{\delta_1\|x_0 - x^*\|^2}$,  $\delta_1=\alpha (1-\frac{1}{12\kappa_d})^{K}$,  the sparsity level  $\tau \geq (16 (12 \kappa_d-1)^2+1)  \tau^*$.
		Then the following inequality holds for the Fed-HT: 
		\begin{align*}
			E[\|x_{T} - x^*\|^2] &\leq  \theta_1^{T} \|x_0 - x^*\|^2+  g_1( x^*).
		\end{align*}
		where $\theta_1 =\omega_1 =  (1+2\alpha) (1-\frac{1}{12\kappa_d})^{K} \in (0,1)$,   $g_1( x^*) =  \frac{\xi_1\mathcal{B}^2 }{1-\psi_1}  \|\nabla f(x^*) \|^2$, $\psi_1 =  (1+\alpha) (1-\frac{1}{12\kappa_d})^{K}$, $\xi_1=\frac{(1+ \alpha)(1 - (1-  \frac{1}{12\kappa_d})^K) \kappa_d }{  l_d^2}$, and $\alpha = \frac{2 \sqrt{\tau^*}}{\sqrt{\tau - \tau^*}}$.
	\end{theorem}

	Note that if the sparse solution $x^*$ is sufficiently close to an unconstrained minimizer of $f(x)$, then $\|\nabla f(x^*) \|$ is small, so the first exponential term on the right-hand side can be a dominating term which approaches to $0$ when $T$ goes to infinity. 
	
	\begin{corollary}\label{cor1}
		If all the conditions in  Theorem \ref{theorem1}  hold, for a given precision $\epsilon >0$, we need at most $T\leq C_1 \log( \frac{\| x_0 - x^*\|}{\epsilon})$ rounds to obtain 
		$$ E[\|x_T - x^*\|^2] \leq  \epsilon +  g_1( x^*),$$
		where $C_1 = - ( \log( \theta_1))^{-1}$. 
	\end{corollary}
	
	Corollary \ref{cor1} indicates that under proper conditions and with sufficient rounds, the estimation error of the Fed-HT is determined by the second term - the statistical bias term - which we denote as $g_1(x^*)$. The term $g_1(x^*)$ can become small if $x^*$ is sufficiently close to an unconstrained minimizer of $f(x)$, so it represents the sparsity-induced bias to the solution of the unconstrained optimization problem.
	The upper bound result guarantees that the Fed-HT can approach $x^*$ arbitrarily closely under a sparsity-induced bias, and the speed of approaching to the biased solution is linear (or geometric) and determined by $\theta_1$. In Theorem \ref{theorem1} and Corollary \ref{cor1}, $\theta_1$ is closely related to the number of local updates $K$. The condition number $\kappa_d >1$, so $(1-\frac{1}{12\kappa_d}) < 1$. When $K$ is larger, $\theta_1$ is smaller, so is the number of rounds $T$ required for reaching a target $\epsilon$. In other words, the Fed-HT converges faster with fewer communication rounds. However, the bias term $g_1( x^*)$ will increase when $K$ increases. Therefore, $K$ should be chosen to  balance the convergence rate and statistical bias.

	We further investigate how the objective function $f(x)$ approaches to the optimal $f(x^*)$ in the following corollary. Detailed proof can be found in supplemental material\footnote{ Supplementary material: \url{https://www.dropbox.com/sh/c75nni6uc5fzd70/AADpB6QoPR0sxPFqO_No-sXKa?dl=0}\label{footnote}}.
	
	\begin{corollary}\label{cor11}
		If all the conditions in Theorem \ref{theorem1} hold, let $\Delta_1 = l_d \|x_0 - x^*\|^2$, and $g_2( x^*)= O(\|\nabla f(x^*) \|^2)$, we have
		$$
		E[f( x_{T}) - f( x^*)]   \leq \theta_1^{T}\Delta_1 + g_2(x^*).
		$$
	\end{corollary}

	Because the local updates on each device are based on stochastic gradient descent with dense parameter, without hard thresholding operator, $l_d$-smoothness and $\rho_d$-strongly convexity are required, which are stronger requirements for $f$. What's more, $\|\nabla f(x^*)\| \leq d\| f(x^*)\|_{\infty}$, which means $g_1(x^*)$ and $g_2(x^*)$  are $O(d^2\| f(x^*)\|_{\infty}^2)$, which are suboptimal, comparing with results for traditional IHT methods, in terms of dimension $d$. In order to solve such drawbacks, we develop a new algorithm in next section.

	\section{The FedIter-HT Algorithm}
	
	If we apply the HT operator to each local update as well, we obtain the FedIter-HT algorithm as described in Algorithm \ref{alg:FedAvgHT2}. Hence, the local update on each device performs multiple SGD-HT steps, which further reduces the communication cost because model parameters sent back from clients to the central server are also sparse. If a client has a communication bandwidth so small that it can not effectively pass the full set of parameters, the FedIter-HT provides a good solution. 
	
	\begin{algorithm}[H]
		\captionof{algorithm}{Federated Iterative Hard Thresholding (FedIter-HT) }
		\label{alg:FedAvgHT2}
		\begin{algorithmic}
			\State {\bfseries Input:} The learning rate $\gamma_t$, the sparsity level $\tau$, and the number of clients $N$. 
			\State {\bfseries Initialize} $x_{0}$ 
			\For {$t = 0$ to $T-1$}
			\For {client $i=1$ to $N$ parallel}
			\State $x_{t,1}^{(i)}= x_{t}$
			\For {$k = 1$ to $K $}
			\State Sample uniformly a batch $ I_{t,k}^{(i)} $
			\State $g_{t,k}^{(i)}= \nabla f_{I_{t,k}^{(i)}}(x_{t,k}^{(i)})$
			\State $x_{t,k+1}^{(i)} = \mathcal{H}_\tau( x_{t,k}^{(i)} - \gamma_t g_{t,k}^{(i)}) $
			\EndFor
			\EndFor
			\State Exact-Average: $x_{t+1} = \mathcal{H}_\tau( \sum_{i=1}^N p_i x_{t,K}^{(i)})$  
			\EndFor
		\end{algorithmic}
	\end{algorithm}
	
	We again examine the convergence of the FedIter-HT by developing an upper bound on the distance between the estimator $x_{T}$ and the optimal $x^*$, i.e. $E[\| x_{T} - {x}^{*}\|^{2}]$ in the following theorem. Then, the detailed proof can be found in supplemental material.
	
	\begin{theorem}\label{theorem3}
		
		Let $x^*$ be  the optimal solution to  (\ref{Problem_sl}),  $\tau^*=\|x^*\|_0 $, and suppose $f( x )$ satisfies Assumptions 1 and 2. The condition number $\kappa_d = \frac{l_d}{\rho_d}\geq 1$.
		Let stepsize $\gamma_t = \frac{1}{6l_s}$ and the batch size $b_t =\frac{\Gamma_2}{\omega_2^t}$, $\Gamma_2 \geq \frac{\xi_2 \sum_{i=1}^N p_i \sigma_i^2}{\delta_2\|x_0 - x^*\|^2}$, $\delta_2=(2\alpha+2\alpha^2 )(1-\frac{1}{12\kappa_s})^{K}$, the sparsity level $\tau \geq (\frac{16}{( \sqrt{\frac{12 \kappa_d}{ 12\kappa_d -1}}-1 )^2}+1)  \tau^*$.
		Then the following inequality holds for the FedIter-HT: 
		\begin{align*}
			E[\|x_{T} - x^*\|^2] &\leq  \theta_2^{T} \|x_0 - x^*\|^2+  g_3( x^*).
		\end{align*}
		where $\theta_2 =\omega_2 =  (1+2\alpha)^2(1-\frac{1}{12\kappa_s})\in (0,1)$,   $g_3( x^*) =  \frac{\xi_2\mathcal{B}^2 }{1-\psi_2}  \|\pi_{\mathcal{\tilde{I}}}( \nabla f(x^*)) \|^2 $ , $\xi_2 = \frac{(1+\alpha)^2(1 - (1-  \frac{1}{12\kappa_s})^K) \kappa_s }{  l_s^2}$, $\psi_2 = (1+\alpha)^2(1-\frac{1}{12\kappa_s})$,  $\alpha = \frac{2 \sqrt{\tau^*}}{\sqrt{\tau - \tau^*}}$, $\widetilde{\mathcal{I}^{i}}=supp\left(\mathcal{H}_{2 \tau}\left(\nabla f_i\left(\boldsymbol{x}^{*}\right)\right)\right) \cup supp\left(\boldsymbol{x}^{*}\right) $ and $\widetilde{\mathcal{I}}=supp\left(\mathcal{H}_{2N \tau}\left(\nabla f\left(\boldsymbol{x}^{*}\right)\right)\right) \cup supp\left(\boldsymbol{x}^{*}\right) $.
	\end{theorem}

	The factor $\theta_2$, 
	compared with $\theta_1$ in Theorem \ref{theorem1}, is smaller if $2\alpha = \frac{4 \sqrt{\tau^*}}{\sqrt{\tau - \tau^*}} \leq (\frac{ 1-1/12\mathcal{K}_d}{1-1/ 12\mathcal{K}_s})^K-1$, which means that the FedIter-HT converges faster than the Fed-HT when the beforehand-guessed sparsity $\tau$ is much larger than the true sparsity. Both $\theta_2$ and $\theta_1$ will decrease when the number of internal iterations $K$ increases, but $\theta_2$ decreases faster than $\theta_1$ because $1-\frac{1}{12 \kappa_s}$ is smaller than $1-\frac{1}{12 \kappa_d}$. Thus, the FedIter-HT is more likely to benefit by increasing $K$ than the Fed-HT.
	The statistical bias term $g_3(x^*)$ can be much smaller than $g_1(x^*)$ in  Theorem \ref{theorem1} because $g_3(x^*)$ only depends on the norm of $\nabla f(x^*)$ restricted to the support $\widetilde{\mathcal{I}}$ of size $2N \tau + \tau^*$. Because the norm of the gradient is a dominating term in $g_1$ and $g_3$, slightly increasing $K$ does not vary much the statistical bias terms (when $d \gg 2N \tau + \tau^*$). 
	
	Using the results in Theorem \ref{theorem3}, we can further derive Corollary \ref{cor2} to specify the number of rounds required to achieve a given estimation precision.
	
	\begin{corollary}\label{cor2}
		If all the conditions in Theorem \ref{theorem3}  hold,  for a given $\epsilon >0$, the FedIter-HT requires the most $T\leq C_2 \log( \frac{\| x_0 - x^*\|}{\epsilon})$ rounds to obtain 
		$$
		E[\|x_T - x^*\|^2] \leq  \epsilon +  g_3( x^*),
		$$
		where $C_2 = - ( \log( \theta_2))^{-1}$. 
	\end{corollary}
	Because $g_3(x^*) = O(\|\pi_{\mathcal{\tilde{I}}}( \nabla f(x^*)) \|^2)$, and we also know $\|\pi_{\mathcal{\tilde{I}}}( \nabla f(x^*)) \|^2 \leq (2N\tau + \tau^*)^2\|\nabla f(x^*)\|_{\infty}^2$ and $2N\tau + \tau^* \ll d$ in high dimensional statistical problems, the result in Corollary \ref{cor2} gives a tighter bound than the one obtained in Corollary \ref{cor1}. Similarly, we also obtain a tighter upper bound for the convergence performance of the objective function $f(x)$.
	\begin{corollary}\label{cor22}
		If all the conditions in  Theorem \ref{theorem3}  hold, let $\Delta_2 = l_s \|x_0 - x^*\|^2$, and $g_4( x^*)= O( \|\pi_{\mathcal{\tilde{I}}}( \nabla f(x^*)) \|^2)$, we have
		$$
		E[f( x_{T}) - f( x^*)]  \leq \theta_2^{T}\Delta_2 + g_4( x^*).$$
	\end{corollary}
	
	The theorem and corollaries developed in this section only depend on the $l_s$-restricted smoothness and $\rho_s$-restricted strong convexity, where $s= 2\tau + \tau^*$, which are the same conditions used in the analysis of existing IHT methods. Moreover, $\|\pi_{\mathcal{\tilde{I}}}(\nabla f(x^*))\| \leq (2N\tau + \tau^*)\|\nabla f(x^*)\|_{\infty}$, which means  $g_3(x^*)$ and $g_4(x^*)$  are $O((2N\tau + \tau^*)^2\| \nabla f(x^*)\|_{\infty}^2)$, where $2N\tau + \tau^*$ is the size of support $\mathcal{\tilde{I}}$; Therefore, our results match the current best known upper bound for the statistic bias term, comparing with the results for traditional IHT methods. 

	\subsection{Statistical analysis for M-estimators}
	Because of the good property of the FedIter-HT, we also  demonstrate the theory of constrained M-estimators obtained on more concrete learning formulations. Although we focus on the sparse linear regression and sparse logistic regression in this paper, our method can be used to analyze other statistical learning problems as well. 
	
	{\noindent \textbf{Sparse Linear Regression.}}
	We consider the linear regression problem in high-dimensional regime:
	\begin{align*} 
		\min_{x\in \mathbb{R}^d}f(x) &=  \frac{1}{N}\sum_{i=1}^N \frac{1}{B} \|Y^{(i)} - Z^{(i)}x \|_2^2 ,\\
		&\text{subject to} \;\; \| x \|_0 \leq \tau, 
	\end{align*}
	where $Z^{(i)} \in \mathbb{R}^{B\times d}$ is a design matrix associated with client $i$. For each row of matrix $Z^{(i)}$, we further assume that they are independently drawn from a sub-Gaussian distribution with parameter $\beta^{(i)}$, $Y^{(i)} = Z^{(i)}x^*  + \epsilon^{(i)}$ denotes the response vector, and $\epsilon^{(i)} \in \mathbb{R}^{B} $ is a noise vector following Normal distribution $N(0, \sigma^2I)$, $x^*\in \mathbb{R}^d$ with $\|x^*\|_0=\tau^*$ is the underlying sparse regression coefficient vector.
	\begin{corollary}\label{cor:LR}
		If all the conditions in  Theorem \ref{theorem3}  hold, with$B \geq C_1 \tau  \log(d) \max_i\{(\beta^{(i)})^2\}$ and a sufficiently large number of communication rounds $T$, we have
		\begin{align*}
			E[\|x_{T} - x^*\|^2 ]&\leq O(\frac{(2N\tau + \tau^*) \sigma^2 \mathcal{B}^2(\sum_{i=1}^N \beta^{(i)})^2 \log(d)}{NB})
		\end{align*}
		with probability at least $(1-\exp(-C_5NB))$, where $C_5$ is a universal constant.
	\end{corollary}
	
	\textit{Proof Sketch}:
	First, we are able to show that $f_i$ is restricted $\rho_s$-strongly convex and  restricted $l_s$-strongly smooth with $\rho_s = \frac{4}{5} $ and $l_s = \frac{6}{5}$ respectively with probability at least $ (1 - \exp(-C_2 B)) $ if the  sample size $B \geq C_1 \tau  \log(d) \max_i\{(\beta^{(i)})^2\}$, where $C_1$ and $C_2 $ are  constants. Secondly, we know that $\|\nabla f(x^*)\|_{\infty} = \|\frac{Z^T\epsilon}{NB}\|_{\infty} \leq C_3 \sigma\sum_{i=1}^N \beta^{(i)}\sqrt{\frac{\log(d)}{NB}}$, with probability at least $(1-\exp(-C_4NB))$, 
	where $C_3$ and $C_4$ are constants irrelevant to the model parameters. Let the number of rounds $T$ be sufficiently large such that the term $\theta_2^{T} \|x_0 - x^*\|^2$ in Theorem \ref{theorem3} is sufficiently small. Gathering everything  together and putting them into the statistical bias term yield the above bound with a high probability.
	\qed
	
	{\noindent \textbf{Sparse Logistic Regression.}}
	We consider the following optimization problem for logistic regression:
	\begin{align*}
		\underset{x}{min} f( x )& = \frac{1}{N}\sum_{i=1}^{N}\frac{1}{B}\sum_{j=1}^{B}(\log (1+\exp(z_{i,j}^{T}x)) - y_{i,j}z_{i,j}^{T}x)\\
		&\text{subject to } \;\; \| x \|_0 \leq \tau, 
	\end{align*}
	where $z_{i,j}\in\mathbb{R}^d$ for $j\in [B]$ is a predictive vector and drawn from a sub-Gaussian distribution associated with client $i$, each observation $y_{i,j}$ on client $i$ is drawn from the Bernoulli distribution $\mathbb{P}( y_{i,j}| z_{i,j}, x^* )= \frac{\exp(z_{i,j}^{T}x^*)}{1+ \exp( z_{i,j}^{T}x^*)}$,  and $x^* \in \mathbb{R}^d$ with $\|x^*\|_0=\tau^*$ is the underlying true parameter that we want to recover.
	\begin{corollary}
		If all the conditions in  Theorem \ref{theorem3}  hold,  $\|z_{i,j}\|\leq \mathcal{K}$, $C_{lower} \leq \exp(z_{i,j}^{T}x)/(1+\exp(z_{i,j}^{T}x))^2 \leq C_{upper}$ for $ i \in [N]$ and $j \in [B]$ and $B \geq C_7 \tau \mathcal{K}^2 log(d)$ and with a sufficiently large number of communication rounds $T$, we have
		\begin{align*}
			E[\|x_{T} - x^*\|^2] &\leq O(\frac{(2N\tau + \tau^*) \mathcal{B}^2\mathcal{K}^2 \log(d)}{NB})
		\end{align*}
		with  probability at least $(1-\exp(-C_6NB) -C_9
		\exp(-C_{10}log( d)) + \frac{C_9}{\exp(C_6NB) \exp(C_{10}log( d))})$, where $C_6$ , $C_{9}$ and $C_{10}$ are constants.
	\end{corollary}
	\textit{Proof Sketch}:
	We have the above result for sparse logistic regression, if we follow the similar argument to that in Corollary \ref{cor:LR}, except that we have $\rho_s = \frac{4}{5} C_{lower} $ and $l_s = \frac{6}{5} C_{upper}$ with a probability at least $ (1 - \exp(-C_6 B)) $ if $B \geq C_7 \tau \mathcal{K}^2 log(d)$, and $\|\nabla f(x^*)\|_{\infty}  \leq C_8 \mathcal{K}\sqrt{\log(d)/NB}$ with a probability at least $(1-C_9exp(-C_{10}log( d))$, where  $C_{lower}$,  $C_{upper}$, $C_l$ for $l\in \{6, ..., 11\}$ are some constants irrelevant to model parameters.\qed

	\section{Experiments}
	We empirically evaluate our methods in both simulations and in the analysis of three real-world datasets (E2006-tfidf, RCV1 and MNIST, see Figure \ref{subfig3}, \ref{subfig4} and Table \ref{table2}, which are downloaded from the LibSVM website\footnote{http://www.csie.ntu.edu.tw/ cjlin/libsvmtools/datasets/}), and compare them against a baseline method. The baseline method is a standard Distributed IHT and communicates every local update to the central server, which then aggregates and broadcasts back to clients (see Supplementary for more detail). Specifically, experiments for simulation I and the E2006-tfidf dataset are done for sparse linear regression. 
	In simulation II and for the RCV1 dataset, we solve the sparse logistic regression problem.
	The last experiment uses MNIST data in a multi-class softmax regression problem. The detailed loss functions for the different problems can be found in Supplementary.
	
	We use Distributed-IHT  as a baseline. Following the convention in the Federated Learning literature, we use the number of communication rounds to measure the communication cost. For a comprehensive comparison, we also include the number of iterations. For both synthetic and real-world datasets, parameters, such as local iterations $K$, 
	stepsize $\gamma$, are determined by the following criteria.  The number of local iterations $K$ is searched from $\{3, 5, 8, 10\}$. 
	The stepsize $\gamma$ for each algorithm is set by a grid search from $\{ 10, 1, 0.6, 0.3, 0.1, 0.06, 0.03, 0.01, 0.001\}$. All the algorithms are initialized with $x^{(0)} = 0$. The sparsity $\tau$ is 500 for MNIST dataset and 200 for all others.
	
	\begin{table}[h!]
		\caption{Statistics of three real federated datasets.}
		\begin{center}
			\begin{tabular}{lllll}
				\hline Dataset & Samples & dimension& \multicolumn{2}{l} { Samples/device } \\
				\cline { 4- 5 } & & & mean & stdev \\
				\hline
				E2006-tfidf & 3,308 & 150,360 &33.8 & 9.1 \\
				RCV1 & 20,242 & 47,236& 202.4& 114.5 \\
				MNIST & 60,000 & 784 & 600 & --\\
				\hline
			\end{tabular}
		\end{center}
		\label{table2}
	\end{table}
	
	\subsection{Simulations}
	To generate synthetic data, we follow a similar setup in \cite{li2018federated}. In simulation I, for each device $i \in [100]$, we generate samples $(z_{i,j}, y_{i,j})$ for $j \in [100]$ according to $y_{i,j} =  z_{i,j}^T x_i + b_{i,j}$, where $z_{i,j}\in \mathbb{R}^{1000}$, $x_i \in \mathbb{R}^{1000}$. The first 100 elements of $x_i$ are IID drawn from $\mathcal{N}( u_i, 1)$ and the remaining elements in $x_i$ are zeros,  $b_{i,j}\sim \mathcal{N}( u_i, 1)$,  $u_i \sim \mathcal{N}( 0.1, \alpha)$, $z_{i, j} \sim \mathcal{N}( v_{i}, \Sigma)$, where $\Sigma$ is diagonal matrix with the $i$-th diagonal element equal to $\frac{1}{i^{1.2}}$. Each element in the mean vector $v_{i}$ is drawn from $\mathcal{N}(B_i, 1)$, $B_i \sim \mathcal{N}(0, \beta)$. Therefore, $\alpha$ controls how much the local models differ from each other,
	and $\beta$ controls how much the local on-device data differ between one device or another. In simulation I, $\alpha = 0.1$ and $\beta=0.1$. The data generation procedure for simulation II is the same as the procedure of simulation I, except that $y'_{i,j} =  \exp( z_{i,j}^T x_i + b_{i,j})/(1+\exp( z_{i,j}^T x_i + b_{i,j}))$, then for the $i$-th client, we set $y_{i,j}=1$ corresponding to the top 100 of $y'_{i,j}$ for $j\in [1000]$, otherwise $y_{i,j}=0$. In simulation II, we set $\alpha = 1$ and $\beta =1$.
	
	\begin{figure}[t]
		\centering
		\includegraphics[width=\linewidth]{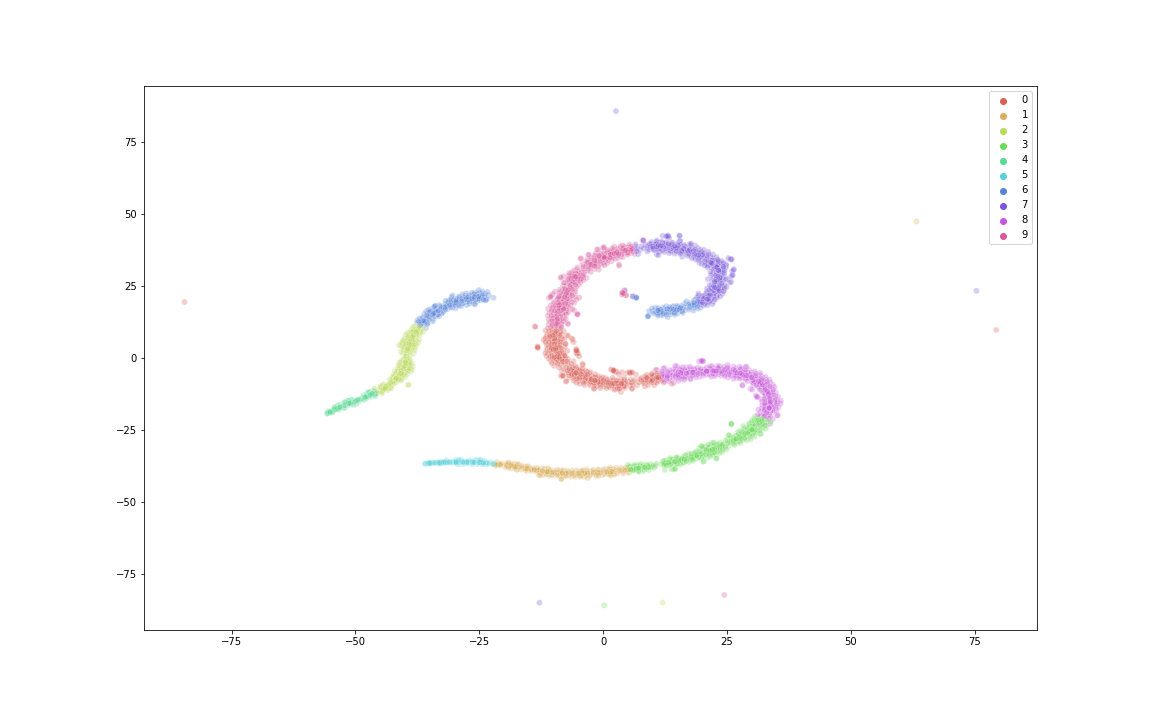}
		\caption{Visualization of labeling with K-means clustering for E2006}
		\label{subfig3}
	\end{figure}
	
	\begin{figure}[t]
		\centering
		\includegraphics[width=\linewidth]{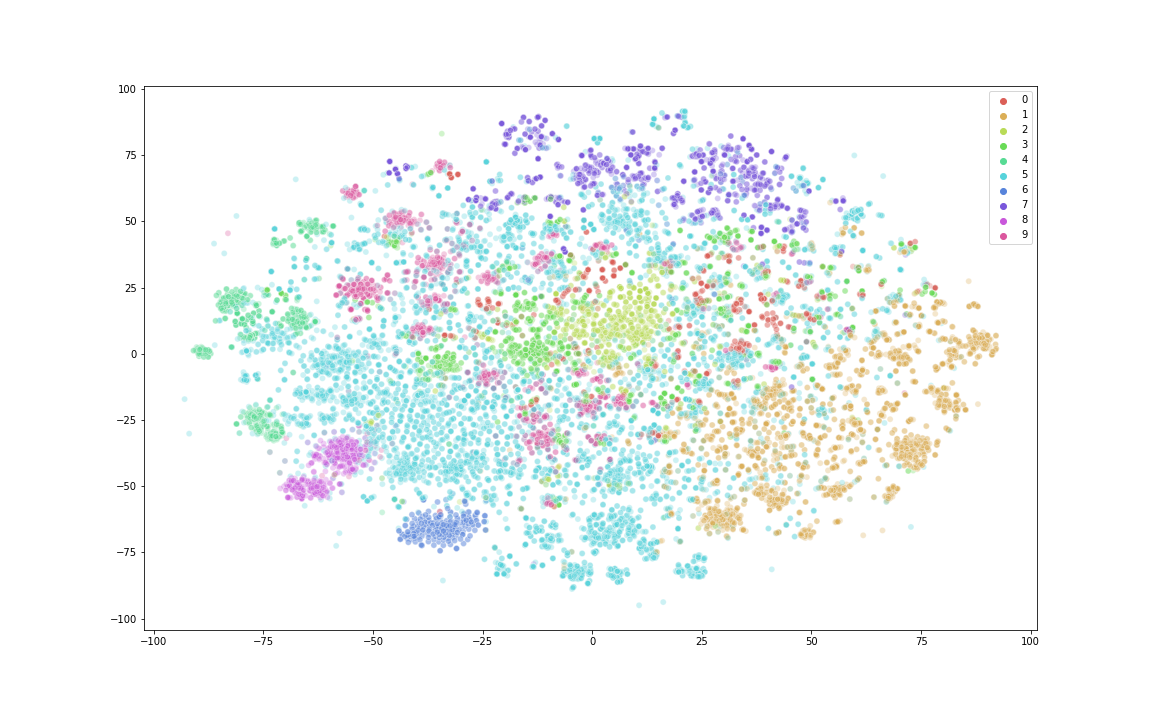}
		\caption{Visualization of labeling with K-means clustering for RCV1}
		\label{subfig4}
	\end{figure}

	\subsection{Benchmark Datasets}
	\begin{figure*}[t]
		\centering
		\includegraphics[width= \linewidth]{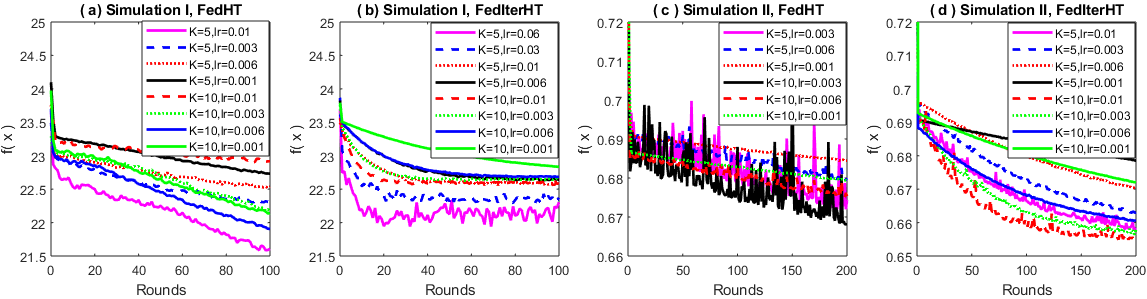}
		\caption{The objective function value vs. communication rounds for regression (a, b) and classification (c, d), and for Fed-HT (a, c) and FedIter-HT (b, d) with varying values of $K$ and stepsize/learning rate (lr) $\gamma$. }
		\label{fig:simulation1}
	\end{figure*}
	
	\begin{figure*}[t]
		\centering
		\includegraphics[width=\linewidth]{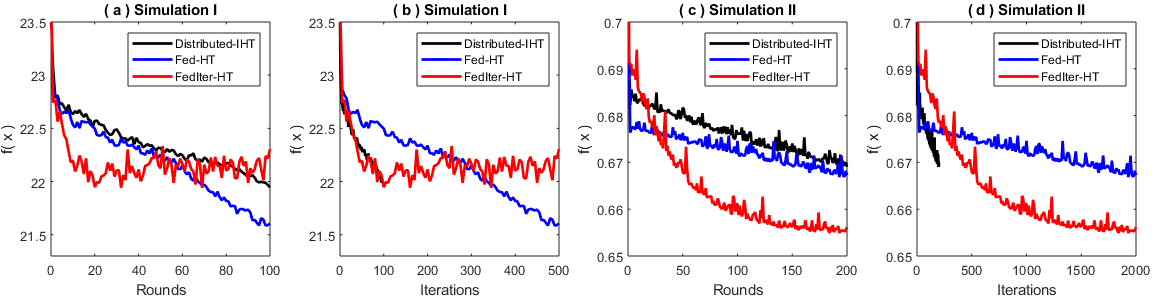}
		\caption{The comparison of different algorithms in terms of the objective function value vs. communication rounds (a, c) and vs. all internal iterations (b, d), and for regression (a, b) and classification (c, d). Note that the distributed IHT is the baseline method that communicates every local update (so the number of rounds equals the number of iterations) and may be the best scenario for reducing the objective value. We observe that in simulation I, the Fed-HT and FedIter-HT only need, respectively, 60 ($\sim 1.7\times$ less) and 20 ($\sim 5\times$ less) communication rounds  to reach the same objective value that the Distributed-IHT takes 100 rounds; in simulation II, the FedIter-HT needs 50 communication rounds ($\sim 4\times$ less) to achieve the same objective value that the Distributed-IHT takes 200 rounds. Although the proposed methods use more internal iterations in (b,d) than that of the Distributed-HT, they are at least 1.6 times faster due to the communication efficiency, if we further assume that clients can be anywhere around the world, for which the average network delay is about 150 ms, whereas the local computation may only take 20 us.}
		\label{fig:simulation2}
	\end{figure*}
	
	\begin{figure*}[t]
		\centering
		\centering
		\includegraphics[width=0.8\linewidth]{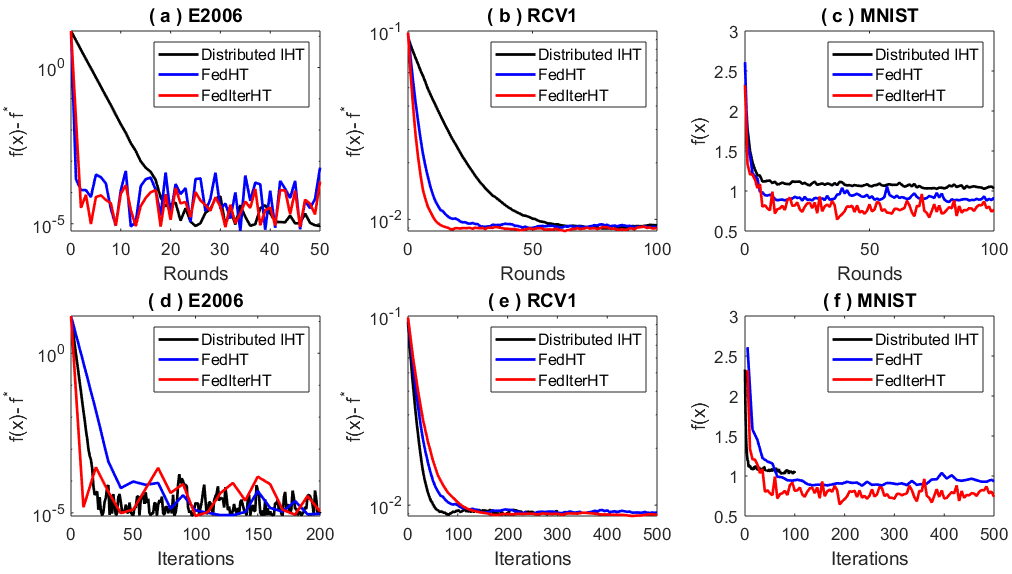}
		\caption{Comparison of the algorithms on different datasets in terms of the objective function value vs. communication rounds (top) and vs. all internal iterations (bottom). $f^*$ is a lower bound of $f(x)$. FedIter-HT performs consistently better across all datasets, which confirms our theoretical result.} 
		\label{fig:realworld}
	\end{figure*}
	
	We use the E2006-tfidf dataset  \cite{kogan2009predicting} to predict the volatility of stock returns based on the SEC-mandated financial text report, represented by tf-idf. It was collected from thousands of publicly traded U.S. companies, for which data from different companies are inherently non-identical and the privacy consideration for financial data demands federated learning. The RCV1 dataset  \cite{lewis2004rcv1} is used to predict categories of newswire stories recently collected by Reuters. Ltd. The RCV1 can be naturally partitioned based on news category and used for federated learning experiments, since readers may only be interested in one or two categories of news and the model training process will mimic the personalized privacy-preserving news recommender system, for which reader history is located on a user's personal devices. For these two datasets, we first run K-means to obtain 10 clusters and use t-SNE to reveal the hidden structures we find with the clustering method. We use the digits to label the MNIST images. Then for all datasets, the data in each category are evenly partitioned into 20 parts, and each client randomly picks 2 categories and selects one part from each of the categories. Because the MNIST images are evenly collected for each digit, the partitioned decentralized MNIST data are balanced in terms of categories, whereas the other two datasets are unbalanced.

	Figure \ref{fig:realworld} (top) shows that the proposed  Fed-HT and FedIter-HT can significantly reduce the communication rounds required to achieve a given accuracy, though they take the cost of running additional internal iterations as shown in Figure \ref{fig:realworld} (bottom).  In Figure \ref{fig:realworld} (a,c), we further observe that federated learning displays more  randomness, when approaching to the optimal solution. This may be caused by dissimilarity across clients. 
	For instance, the three different algorithms in Figure \ref{fig:realworld} (c) reach the neighborhood of different solutions at the end where the proposed FedIter-HT obtains the lowest objective value. These behaviors may be worth further exploring in the future.

	\section{Conclusion}
	In this paper, we propose two communicate-efficient IHT methods - Fed-HT and FedIter-HT - to deal with nonconvex sparse learning with decentralized non-IID data. The Fed-HT algorithm is designed to impose a hard thresholding operator at a central serve, whereas the FedIter-HT applies this operator at each update no matter at local clients or a central server. Both methods reduce communication costs - in both the communication rounds and the communication load at each round. Theoretical analysis shows a linear convergence rate for both of the algorithms where the Fed-HT has a better reduction factor $\theta$ in each iteration but the FedIter-HT has a better statistical estimation bias. Even with the decentralized non-IID data, there is still a guarantee to recover the optimal sparse estimator, in a similar way to the traditional IHT methods with IID data.  Empirical results demonstrate that they outperform the standard distributed IHT in simulations and on benchmark datasets.

	\bibliographystyle{IEEEtran}
	\bibliography{FedLearning} 
	\newpage
	\onecolumn
	\appendix

	\section{Distributed IHT Algorithm}
	\begin{algorithm}[H]
		\captionof{algorithm}{Distributed-IHT}
		\label{alg:distributed-IHT}
		\begin{algorithmic}
			\State {\bfseries Input:} Learning rate $\gamma_t$, number of workers $N$.
			\State {\bfseries Initialize} $x_{0}$
			\For {$t = 0$ to $T-1$}
			\For {worker $i=1$ to $N$ parallel}
			\State Receive $x_{t}^{(i)} = x_{t}$ from the central server
			\State Calculate unbiased stochastic gradient direction $v_{t}^{(i)}$ on worker $i$
			\State  Locally update: $x_{t+1}^{(i)} = x_{t}^{(i)} - \gamma_t v_{t}^{(i)} $ 
			\State Send $x_{t+1}^{(i)}$ to the central server
			\EndFor
			\State Receive all local updates and average on remote server: $x_{t+1} = \mathcal{H}_\tau( \sum_{i=1}^N p_i x_{t+1}^{(i)})$  
			\EndFor
		\end{algorithmic}
	\end{algorithm}
	
	\section{More Experiment Details}
	In more detail, experiments for simulation I and real data  E2006-tfidf dataset are done with sparse linear regression, 
	\begin{equation*}
		\min_{x\in \mathbb{R}^d}f(x) =  \frac{1}{N}\sum_{i=1}^N \frac{1}{B^{(i)}} \|Y^{(i)} - Z^{(i)}x \|_2^2 , \;\;\;\text{subject to} \;\; \| x \|_0 \leq \tau.
	\end{equation*}
	Experiments for simulation II and real data RCV1 are done with sparse logistic regression
	\begin{align*}
		\underset{x}{\min} f( x ) = \frac{1}{N}\sum_{i=1}^{N}\frac{1}{B^{(i)}}\sum_{j=1}^{B^{(i)}}(\log (1+exp(y_{i,j}z_{i,j}^{T}x)) + \frac{\lambda}{2}\|x\|^2),
		\;\;\;\text{subject to } \;\; \| x \|_0 \leq \tau.
	\end{align*}
	The last experiment is for MNIST data with multi-class softmax regression problem as follows:
	\begin{align*}
		&\underset{x}{\min} \{f( x ) = \frac{1}{N}\sum_{i=1}^N\frac{1}{B^{(i)}}\sum_{j=1}^{B^{(i)}}( \sum_{r=1}^{c}(-\mathbb{I} ( y_{i,j} = r )\log( \frac{\exp( z_{i,j}^T x_r)}{\sum_{l=1}^{c}\exp( z_{i,j}^Tx_l)}) + \frac{\lambda}{2}\|x_r\|^2))\}, \\&\text{subject to } \;\; \| x_r \|_0 \leq \tau ,\;\; \forall  r \in \{1,2,...,c\}.\nonumber
	\end{align*}
	
	\begin{lemma} 
		A differentiable convex function $f_i( x ): \mathbb{R}^d \rightarrow \mathbb{R}$ is restricted $l_s$-strongly smooth with parameter s, i.e. there exists a generic constant $L_s > 0$ such that for any $x_1$, $x_2$ with $\| x_1 - x_2\|_0 \leq s $ and 
		$$f_i( x_1) - f_i( x_2) - \langle\nabla f_i( x_2), x_1 - x_2\rangle \leq \frac{L_s}{2}\| x_1 - x_2 \|^2,$$ then we have:
		$$\|\nabla f_i( x_1) - \nabla f_i( x_2)\|^2 \leq 2 l_s ( f_i( x_1) - f_i( x_2) + \langle \nabla f_i( x _2), x_2 - x_1\rangle).$$
		This is also true for global smoothness parameter $l_d$.
	\end{lemma}
	
	\begin{proof}
		Let $\phi ( y ) = f_i( y) - \langle \nabla f_i( x ), y \rangle $, then $\phi( y)$ is restricted $l_s$-strongly smooth with parameter $s$ too.
		
		\begin{align}
			\phi ( x ) & = \min_v \phi( v )  \label{eq:6}\\
			& \leq \min_v \{ \phi{(y)}  + \langle \nabla \phi( y), v- y \rangle + \frac{l_s}{2}\|v - y\|^2 \} \label{eq:7}\\
			& = \phi{(y)} - \frac{1}{2L_s} \| \nabla \phi( y ) \|^2
		\end{align}
		where the equality (\ref{eq:6}) is due to $\nabla \phi( x ) = 0$; inequality (\ref{eq:7}) is due to restricted $l_s$-strongly smoothness.
		
		Let $y = x_1$ and $x = x_2$ and reorganize, we have 
		$$\|\nabla f_i( x_1) - \nabla f_i( x_2)\|^2 \leq 2 l_s ( f_i( x_1) - f_i( x_2) + \langle \nabla f_i( x _2), x_2 - x_1\rangle).$$
		
		Also, for global smoothness parameter $l_d$, we have
		$$\|\nabla f_i( x_1) - \nabla f_i( x_2)\|^2 \leq 2 l_d ( f_i( x_1) - f_i( x_2) + \langle \nabla f_i( x _2), x_2 - x_1\rangle).$$
	\end{proof}
		
	\subsection{Proof of Theorem 3.1.}
	\begin{proof}
		For Fed-HT Algorithm:
		\begin{align}
			E[\|&x_{t+1} - x^*\|^2] = E\|\mathcal{H}_{\tau}(\sum_{i=1}^N p_i x_{t,K}^{(i)}) - x^* \|^2] \nonumber\\
			&\leq (1+\alpha)E[\|\sum_{i=1}^N p_i x_{t,K}^{(i)} - x^* \|^2] \label{eq:h11}\\
			&= (1+\alpha)E[\|\sum_{i=1}^N p_i  x_{t,K}^{(i)} - \sum_{i=1}^N p_i x^* \|^2] \label{eq:pi1}\\
			&\leq (1+\alpha) \sum_{i=1}^N p_iE^{(i)}[\|  x_{t,K}^{(i)} - x^* \|^2] \label{eq:ieq1}.
		\end{align}
		
		Equation (\ref{eq:h11}) holds due to Lemma \ref{lemma:HT},   (\ref{eq:pi1}) holds because $\sum_{i=1}^N p_i = 1$,  (\ref{eq:ieq1}) holds due to Jensen's Inequality and the sampling procedures across different clients are independent to each other.
		
		We calculate the stochastic gradient which is essential in local update, we split stochastic gradient into 3 terms.
		Note that the last inequality holds due to bounded variance on support assumption and the inequality 
		$\|\nabla f_i(x_t) -\nabla f_i(x^*) \|^2 \leq 2l_d (f_i(x_t) - f_i(x^*) + \langle \nabla f_i(x^*), x_t -x^* \rangle)$.
		\begin{align}
			&\sum_{i=1}^N p_iE^{(i)}[\| g_{t,K-1}^{(i)}\|^2 ] \label{eq:6}\\
			&= \sum_{i=1}^N p_i E^{(i)}[\| g_{t,K-1}^{(i)} -  \nabla f_i(x_{t,K-1}^{(i)}) \nonumber \\
			&+ \nabla f_i(x_{t,K-1}^{(i)}) - \nabla f_i(x^*) +  \nabla f_i(x^*)\|^2] \nonumber\\
			&\leq 3\sum_{i=1}^N p_i E^{(i)}[\| g_{t,K-1}^{(i)} - \nabla f_i(x_{t,K-1}^{(i)})\|^2] \nonumber&\\&+ 3 \sum_{i=1}^N p_iE^{(i)}[ \| \nabla f_i(x_{t,K-1}^{(i)}) - \nabla f_i(x^*) \|^2] + 3 \sum_{i=1}^N p_i \|\nabla f_i(x^*) \|^2 \nonumber \\
			&\leq 3\sum_{i=1}^N p_i \frac{\sigma_i^2}{b_t} + 3 \sum_{i=1}^N p_i \|\nabla f_i(x^*) \|^2 \nonumber&\\&+ 6l_d \sum_{i=1}^N p_i  E^{(i)}[( f_i(x_{t,K-1}^{(i)}) - f_i(x^*) + \langle \nabla f_i(x^*), x_{t,K-1}^{(i)} - x^* \rangle)]. \nonumber
		\end{align}

		Next we want to build the connection of $\sum_{i=1}^N p_i E^{(i)}[\| x_{t,K}^{(i)} - x^* \|^2]$ and $\sum_{i=1}^N p_i E^{(i)}[\| x_{t,K-1}^{(i)} - x^* \|^2]$. Let  $\gamma_t = \frac{1}{6l_d}$. Consider the inner loop iteration:
		
		\begin{align*}
			&\sum_{i=1}^N p_i E^{(i)}[\| x_{t,K}^{(i)} - x^* \|^2]\\
			&= \sum_{i=1}^N p_i E^{(i)}[\| x_{t,K-1}^{(i)} - \frac{1}{6l_d} g_{t,K-1}^{(i)} - x^* \|^2] 
		\end{align*}
		
		\begin{align*}
			&= \sum_{i=1}^N p_i E^{(i)}[\| x_{t,K-1}^{(i)} - x^* \|^2] + \frac{1}{36l_d^2} \sum_{i=1}^N p_i E^{(i)}[\|  g_{t,K-1}^{(i)}\|^2 ]\\
			&- \frac{1}{3l_d}\sum_{i=1}^N p_i E^{(i)}[ \langle  x_{t,K-1}^{(i)} - x^*, g_{t,K-1}^{(i)} \rangle]\\
			&\leq \sum_{i=1}^N p_i E^{(i)}[\| x_{t,K-1}^{(i)} - x^* \|^2] + \frac{1}{36l_d^2} \sum_{i=1}^N p_i E^{(i)}[ \| g_{t,K-1}^{(i)}\|^2]\\
			&- \frac{1}{3l_d}\sum_{i=1}^N p_i E^{(i)}[f_i(x_{t,K-1}^{(i)}) - f_i(x^*)].
		\end{align*}
		
		Plug in (\ref{eq:6}), we further derive
		
		\begin{align*}
			&\sum_{i=1}^N p_i E^{(i)}[\| x_{t,K}^{(i)} - x^* \|^2]\\
			&\leq \sum_{i=1}^N p_i E^{(i)}[\| x_{t,K-1}^{(i)} - x^* \|^2] \\& + \frac{1}{36l_d^2}   (3\sum_{i=1}^N p_i \frac{\sigma_i^2}{b_t} + 6l_d \sum_{i=1}^N p_i  E^{(i)}[ f_i(x_{t,K-1}^{(i)}) - f_i(x^*)\\
			&+ { \langle \nabla f_i(x^*), x_{t,K-1}^{(i)} - x^* \rangle]} 
			+ 3 \sum_{i=1}^N p_i \|\nabla f_i(x^*) \|^2 )\\
			&- \frac{1}{3l_d}\sum_{i=1}^N p_i E^{(i)}[f_i(x_{t,K-1}^{(i)}) - f_i(x^*)]\\
			&=\sum_{i=1}^N p_i E^{(i)}[\| x_{t,K-1}^{(i)} - x^* \|^2] + \frac{1}{12l_d^2} \sum_{i=1}^N p_i \frac{\sigma_i^2}{b_t}\\&- \frac{1}{6l_d}\sum_{i=1}^N p_i E^{(i)}[f_i(x_{t,K-1}^{(i)}) - f_i(x^*)]\\ &+\frac{1}{6l_d}\sum_{i=1}^N p_i E^{(i)}[\langle \pi_I(\nabla f_i(x^*)), x_{t,K-1}^{(i)} - x^* \rangle] \\& + \frac{1}{12l_d^2} \sum_{i=1}^N p_i\|\nabla f_i(x^*) \|^2 \\
			&\leq \sum_{i=1}^N p_i E^{(i)}[\| x_{t,K-1}^{(i)} - x^* \|^2] + \frac{1}{12l_d^2} \sum_{i=1}^N p_i \frac{\sigma_i^2}{b_t} \\&- \frac{1}{6l_d}\sum_{i=1}^N p_i E^{(i)}[\langle \pi_I(\nabla f_i(x^*)), x_{t,K-1}^{(i)} - x^* \rangle \\
			&+ \frac{\rho_d}{2}\| x_{t,K-1}^{(i)} - x^* \|^2]+ \frac{1}{12l_d^2} \sum_{i=1}^N p_i\|\nabla f_i(x^*) \|^2 \\&+\frac{1}{6l_d}\sum_{i=1}^N p_i E^{(i)}[\langle \pi_I(\nabla f_i(x^*)), x_{t,K-1}^{(i)} - x^* \rangle] \\
		\end{align*}
		\begin{align*}
			&= (1-\frac{1}{12\kappa_d})\sum_{i=1}^N p_i E^{(i)}[\| x_{t,K-1}^{(i)} - x^* \|^2 ]\\&+ \frac{1}{12l_d^2} \sum_{i=1}^N p_i \frac{\sigma_i^2}{b_t} + \frac{1}{12l_d^2} \sum_{i=1}^N p_i\|\nabla f_i(x^*) \|^2,
		\end{align*}
		where the last inequality holds due to restricted strongly convexity and $\kappa_d = \frac{l_d}{\rho_d}$. Then for $k \in \{0,1,..., K-1\}$, we have
		\begin{align*}
			&\sum_{i=1}^N p_i E^{(i)}[\| x_{t,K}^{(i)} - x^* \|^2]\\
			&\leq (1-\frac{1}{12\kappa_d})^{K} \sum_{i=1}^N p_i E^{(i)}[\| x_{t,0}^{(i)} - x^* \|^2]
			\\&+\sum_{k=0}^{K-1}(1-\frac{1}{12\kappa_d})^{k} \frac{1}{12l_d^2} \sum_{i=1}^N p_i\frac{\sigma_i^2}{b_t}\\
			&+\sum_{k=0}^{K-1}(1-\frac{1}{12\kappa_d})^{k} \frac{1}{12l_d^2} \sum_{i=1}^N p_i\|\nabla f_i(x^*) \|^2\\
			&\leq (1-\frac{1}{12\kappa_d})^{K} \sum_{i=1}^N p_i E^{(i)}[\| x_{t,0}^{(i)} - x^* \|^2]
			\\&+\sum_{k=0}^{K-1}(1-\frac{1}{12\kappa_d})^{k} \frac{1}{12l_d^2} 
			\sum_{i=1}^N p_i ( \frac{\sigma_i^2}{b_t} + \|\nabla f_i(x^*) \|^2).
		\end{align*}
		
		Let $\psi_1 =  (1+\alpha) (1-\frac{1}{12\kappa_d})^{K}$ and $\xi_1=\frac{(1+ \alpha)(1 - (1-  \frac{1}{12\kappa_d})^K) \kappa_d }{  l_d^2}$. 
		Then, we have
		\begin{align*} 
			&E[\|x_{t+1} - x^*\|^2] \\
			&\leq (1+\alpha)  (1-\frac{1}{12\kappa_d})^{K-1} \sum_{i=1}^N p_iE^{(i)}[ \| x_{t,0}^{(i)} - x^* \|^2]\\
			&+  \frac{(1+ \alpha)(1 - (1-  \frac{1}{12\kappa_d})^K) \kappa_d }{  l_d^2} \sum_{i=1}^N p_i ( \frac{\sigma_i^2}{b_t} + \|\nabla f_i(x^*) \|^2)\\
			&= \psi_1 \sum_{i=1}^N p_iE^{(i)}[ \| x_{t,0}^{(i)} - x^* \|^2] + \frac{\xi_1 \sum_{i=1}^N p_i \sigma_i^2}{b_t} \\
			&+ \xi_1 \sum_{i=1}^N p_i  \|\nabla f_i(x^*) \|^2.
		\end{align*}
		
		Since $x_{t,0} = x_t$, we derive the relation between $\|x_{t+1} - x^*\|^2$ and $\|x_{t} - x^*\|^2$,
		\begin{align*} 
			E[\|x_{t+1} - x^*\|^2] &\leq   \psi_1 E[\|x_t - x^*\|^2]\\&+ \frac{\xi_1 \sum_{i=1}^N p_i \sigma_i^2}{b_t} + \xi_1 \sum_{i=1}^N p_i  \|\nabla f_i(x^*) \|^2.
		\end{align*}
		
		We further set $b_t = \frac{\Gamma_1}{\omega_1^t}$ and assume $\Gamma_1$ is large enough such that
		\begin{align*}
			\upsilon := \frac{\xi_1 \sum_{i=1}^N p_i \sigma_i^2}{\Gamma_1} \leq \delta_1\|x_0 - x^*\|^2,
		\end{align*}
		where $\delta_1$ is a positive constant and will be set later. 
		
		With mathematical induction, we want to prove for $\theta_1 \in (0, 1)$, the following inequality holds.
		\begin{align*} 
			E[&\|x_{t} - x^*\|^2] \\&\leq   \theta_1^t E[\|x_0 - x^*\|^2] + \frac{\xi_1}{1-\psi_1} \sum_{i=1}^N p_i  \|\nabla f_i(x^*) \|^2.
		\end{align*}
		
		When $t=0$, the above inequality is true. Now we assume that for $k = t$,  it holds. Then for $k=t+1$, we have
		\begin{align*}
			&E[\|x_{t+1} - x^*\|^2] \\&\leq   \psi_1 E[\|x_t - x^*\|^2]+ \frac{\xi_1 \sum_{i=1}^N p_i \sigma_i^2}{b_t} + \xi_1 \sum_{i=1}^N p_i  \|\nabla f_i(x^*) \|^2\\
			&\leq \psi_1 E[\|x_t - x^*\|^2]+ \omega_1^t\delta_1\|x_0 - x^*\|^2 + \xi_1 \sum_{i=1}^N p_i  \|\nabla f_i(x^*) \|^2\\
			&\leq  (\psi_1 \theta_1^t + \delta_1 \omega_1^t) E[\|x_0 - x^*\|^2]\\&+ (\frac{\psi_1}{1-\psi_1} + 1)\xi_1 \sum_{i=1}^N p_i  \|\nabla f_i(x^*) \|^2\\
			&\leq  (\psi_1 \theta_1^t + \delta_1 \omega_1^t) E[\|x_0 - x^*\|^2]+ \frac{\xi_1 }{1-\psi_1} \sum_{i=1}^N p_i  \|\nabla f_i(x^*) \|^2.
		\end{align*}
		Let $\theta_1 = \omega = \psi_1 + \delta_1$, we get
		\begin{align*}
			E[&\|x_{t+1} - x^*\|^2] \\&\leq   \theta_1^{t+1} E[\|x_0 - x^*\|^2] + \frac{\xi_1}{1-\psi_1} \sum_{i=1}^N p_i  \|\nabla f_i(x^*) \|^2\\
			&\leq   \theta_1^{t+1} E[\|x_0 - x^*\|^2]  + \frac{\xi_1\mathcal{B}^2 }{1-\psi_1}   \|\nabla f(x^*) \|^2.
		\end{align*}
		Further more, there exists a large $\Gamma_1 \geq \frac{\xi_1 \sum_{i=1}^N p_i \sigma_i^2}{\delta_1\|x_0 - x^*\|^2}$, such that $\delta_1=\alpha (1-\frac{1}{12\kappa_d})^{K}$. Therefore, we have $\omega_1 = \theta_1 = \psi_1 + \delta_1 = (1+2\alpha) (1-\frac{1}{12\kappa_d})^{K} < 1$. Then we can derive the restriction on sparse parameter $\tau \geq (16 (12 \kappa_d-1)^2+1)  \tau^*$.
	\end{proof}

	\subsection{Proof of Corollary 3.1.2.}
	
	\begin{proof}
		In next stage, we use previous upper bound for $E[\|x_{T} - x^*\|^2]$ and $l_d$-restricted strongly smooth conditions to establish epoch-based convergence of $f( x_{T}) - f( x^*) $.
		
		We first use $l_s$-restricted strongly smooth conditions and $\langle a, b\rangle \leq \frac{1}{2} \|a\|^2 + \frac{1}{2} \|b\|^2$ and obtain:
		\begin{align}
			f( x_{T}) &
			\leq f( x^*) +   \langle \nabla f( x^*), x_{T} - x^* \rangle + \frac{l_d}{2} \| x_{T} - x^*\|^2\nonumber \\
			& =  f( x^*) +   ( \langle \nabla f( x^*), x_{T} - x^* \rangle) + \frac{l_d}{2} \| x_{T} - x^*\|^2 \nonumber \\
			& \leq f( x^*) +  \frac{ 1 }{2l_d} \|( \nabla f( x^*))\|^2 \nonumber\\
			& + \frac{l_d}{2  }\|x_{T} - x^* \|^2 + \frac{l_d}{2} \| x_{T} - x^*\|^2 \nonumber \\
			&=  f( x^*) +  \frac{ 1 }{2l_d} \|( \nabla f( x^*))\|^2 + l_d \| x_{T} - x^*\|^2. \nonumber
		\end{align}
		
		Take expectation on both sides, 
		\begin{align*}
			E[ f( x_{T}) - f( x^*)]=   \frac{ 1 }{2l_d} \|( \nabla f( x^*))\|^2 + l_d E[\| x_{T} - x^*\|^2].
		\end{align*}
		
		From the upper bound of $E[\|x_{T} - x^*\|^2]$,
		\begin{align*}
			E[\|x_{T} - x^*\|^2] &\leq  \theta_1^{T} \|x_0 - x^*\|^2+   \frac{\xi_1}{1-\psi_1} \|\nabla f(x^*) \|^2.
		\end{align*}
		
		We can get the final convergence result:
		\begin{align*}
			E[f( x_{T}) - f( x^*)]  & \leq \frac{ 1 }{2l_d} \|( \nabla f( x^*))\|^2 + l_d E[ \| x_{T} - x^*\|^2] \\
			&\leq \theta_1^{T-1} l_d \|x_0 - x^*\|^2+  (  \frac{\xi_1}{1-\psi_1} + \frac{1}{2l_d}) \|\nabla f(x^*) \|^2 \\
			&= \theta_1^{T-1}\Delta_1 + g_2( x^*),
		\end{align*}
		where $\Delta_1 = l_d \|x_0 - x^*\|^2$, $g_2( x^*)=  ( \frac{\xi_1}{1-\psi_1} + \frac{1}{2l_d}) \|\nabla f(x^*) \|^2  = O(\|\nabla f(x^*) \|^2)$.
		
	\end{proof}
	
	\subsection{Proof of the Theorem 4.1.}
	Then we do analysis on our FedIter-HT Algorithm.
	\begin{proof}
		For FedIter-HT Algorithm:
		\begin{align}
			E[\|x_{t+1} - x^*\|^2] &= E[\|\mathcal{H}_{\tau}(\sum_{i=1}^N p_i x_{t,K}^{(i)}) - x^* \|^2] \nonumber\\
			&\leq (1+\alpha)E[\|\sum_{i=1}^N p_i x_{t,K}^{(i)} - x^* \|^2] \label{eq:h13}\\
			&= (1+\alpha)  \sum_{i=1}^N p_i E^{(i)}[\| x_{t,K}^{(i)} - x^* \|^2] \label{eq:ieq3}
		\end{align}
		Eq (\ref{eq:h13}) hold due to Lemma 2.1 , Eq (\ref{eq:ieq3}) holds due to Jensen Inequality. 
		
		This time we calculate the stochastic gradient on support, which is different from the analysis of Fed-HT Algorithm.
		
		We also split stochastic gradient on support into 3 terms,
		\begin{align}
			&\sum_{i=1}^N p_i E^{(i)}[\| \pi_{\mathcal{I}^{(i)}} (g_{t,K-1}^{(i)})\|^2 ]\nonumber\\
			&= \sum_{i=1}^N p_i E^{(i)}[\| \pi_{\mathcal{I}^{(i)}} (g_{t,K-1}^{(i)} -  \nabla f_i(x_{t,K-1}^{(i)})+ \nabla f_i(x_{t,K-1}^{(i)}) - \nabla f_i(x^*) +  \nabla f_i(x^*))\|^2] \nonumber\\
			&\leq 3\sum_{i=1}^N p_i E^{(i)}[\| \pi_{\mathcal{I}^{(i)}} (g_{t,K-1}^{(i)} - \nabla f_i(x_{t,K-1}^{(i)}))\|^2] + 3 \sum_{i=1}^N p_iE^{(i)}[ \| \pi_{\mathcal{I}^{(i)}} (\nabla f_i(x_{t,K-1}^{(i)}) - \nabla f_i(x^*)) \|^2] \nonumber \\
			&+ 3 \sum_{i=1}^N p_i \|\pi_{\mathcal{I}^{(i)}} (\nabla f_i(x^*)) \|^2 \nonumber \\
			&\leq 3\sum_{i=1}^N p_i \frac{\sigma_i^2}{b_t} + 6l_s \sum_{i=1}^N p_i E^{(i)}[ ( f_i(x_{t,K-1}^{(i)}) - f_i(x^*) + \langle \pi_{\mathcal{I}^{(i)}}(\nabla f_i(x^*)), x_{t,K-1}^{(i)} - x^* \rangle)] \nonumber \\
			&+ 3 \sum_{i=1}^N p_i \|\pi_{\mathcal{I}^{(i)}} (\nabla f_i(x^*)) \|^2 \label{eq:g},
		\end{align}
		where the last inequality holds due to bounded variance on support assumption and the inequality 
		$\|\pi_{\mathcal{I}^{(i)}} (\nabla f_i(x_t) -\nabla f_i(x^*)) \|^2 \leq 2l_s (f_i(x_t) - f_i(x^*) + \langle \pi_{\mathcal{I}^{(i)}}(\nabla f_i(x^*)), x_t -x^* \rangle)$.

		Next we want to build the connection of $\sum_{i=1}^N p_i E^{(i)}[\| x_{t,K}^{(i)} - x^* \|^2]$ and $\sum_{i=1}^N p_i E^{(i)}[\| x_{t,K-1}^{(i)} - x^* \|^2]$. Let $\gamma_t = \frac{1}{6l_s}$. Consider the inner loop iteration, 
		\begin{align*}
			\sum_{i=1}^N p_i E^{(i)}[\| x_{t,K}^{(i)} - x^* \|^2]
			&=  \sum_{i=1}^N p_iE^{(i)}[ \|\mathcal{H}_{\tau}( x_{t,K-1}^{(i)} - \frac{1}{6l_s} \pi_{\mathcal{I}^{(i)}} (g_{t,K-1}^{(i)})) - x^* \|^2] \\
			&\leq (1+\alpha)\sum_{i=1}^N p_i E^{(i)}[\| x_{t,K-1}^{(i)} - \frac{1}{6l_s} \pi_{\mathcal{I}^{(i)}} (g_{t,K-1}^{(i)}) - x^* \|^2].
		\end{align*}
		Borrow from the above result that
		\begin{align*}
			&\sum_{i=1}^N p_i E^{(i)}[ \| x_{t,K-1}^{(i)} - \frac{1}{6l_s} \pi_{\mathcal{I}^{(i)}} (g_{t,K-1}^{(i)}) - x^* \|^2]  \\
			&\leq (1-\frac{1}{12\kappa_s})\sum_{i=1}^N p_i E^{(i)}[\| x_{t,K-1}^{(i)} - x^* \|^2] + \frac{1}{12l_s^2} \sum_{i=1}^N p_i \frac{\sigma_i^2}{b_t} + \frac{1}{12l_s^2} \sum_{i=1}^N p_i\|\pi_{\mathcal{I}^{(i)}} (\nabla f_i(x^*)) \|^2,
		\end{align*}
		
		therefore, for $k \in\{0,1,..., K-1\}$, we have
		\begin{align*}
			\sum_{i=1}^N p_i E^{(i)}[\| x_{t,K}^{(i)} - x^* \|^2]
			&\leq (1+\alpha)(1-\frac{1}{12\kappa_s})\sum_{i=1}^N p_i E^{(i)}[\| x_{t,K-1}^{(i)} - x^* \|^2 ]\\
			&+   \frac{(1 - (1-  \frac{1}{12\kappa_s})^K) \kappa_s }{  l_s^2} (\sum_{i=1}^N p_i \frac{\sigma_i^2}{b_t} + \|\pi_{\mathcal{I}^{(i)}} (\nabla f_i(x^*)) \|^2)\\
		\end{align*}
		\begin{align*} 
			E[\|x_{t+1} - x^*\|^2] &\leq  (1+\alpha)^2 ((1-\frac{1}{12\kappa_s})^{K} \sum_{i=1}^N p_i E^{(i)}[\| x_{t,0}^{(i)} - x^* \|^2]
			\\
			&+   \frac{(1+\alpha)^2(1 - (1-  \frac{1}{12\kappa_s})^K) \kappa_s }{  l_s^2} (\sum_{i=1}^N p_i \frac{\sigma_i^2}{b} + \|\pi_{\mathcal{I}^{(i)}} (\nabla f_i(x^*)) \|^2).
		\end{align*}
		Similarly, we have the following result:
		
		\begin{align}
			E[\|x_{t+1} - x^*\|^2] &\leq   \theta_2^{t+1} E[\|x_0 - x^*\|^2] + \frac{\xi_2}{1-\psi_2} \sum_{i=1}^N p_i  \|\pi_{\mathcal{I}^{(i)}} (\nabla f_i(x^*)) \|^2,
		\end{align}
		where $\theta_2=(1+2\alpha)^2(1-\frac{1}{12\kappa_s})$, $\xi_2 = \frac{(1+\alpha)^2(1 - (1-  \frac{1}{12\kappa_s})^K) \kappa_s }{  l_s^2}$, $\psi_2 = (1+\alpha)^2(1-\frac{1}{12\kappa_s})$ and $b_t = \frac{\Gamma_2}{\omega_2^t}$. Further more, there exists a large $\Gamma_2 \geq \frac{\xi_2 \sum_{i=1}^N p_i \sigma_i^2}{\delta_2\|x_0 - x^*\|^2}$, such that $\delta_2=(2\alpha+2\alpha^2 )(1-\frac{1}{12\kappa_s})^{K}$. Therefore, we have $\omega_2 = \theta_2 = \psi_2 + \delta_2 = (1+2\alpha)^2 (1-\frac{1}{12\kappa_d})^{K} < 1$. Then we can derive the restriction on sparse parameter $\tau \geq (\frac{16}{( \sqrt{\frac{12 \kappa_d}{ 12\kappa_d -1}}-1 )^2}+1)  \tau^*$.

	\end{proof}

	\subsection{Proof of Corollary 4.1.2.}
	
	\begin{proof}
		In next stage, we use previous upper bound for $E[\|x_{T} - x^*\|^2]$ and $l_s$-restricted strongly smooth conditions to establish epoch-based convergence of $f( x_{T}) - f( x^*) $.
		
		We first use $l_s$-restricted strongly smooth conditions and $\langle a, b\rangle \leq \frac{1}{2} \|a\|^2 + \frac{1}{2} \|b\|^2$ and obtain:
		\begin{align}
			f( x_{T}) &
			\leq f( x^*) +   \langle \nabla f( x^*), x_{T} - x^* \rangle + \frac{l_s}{2} \| x_{T} - x^*\|^2\nonumber \\
			& =  f( x^*) +   \pi_{\mathcal{\tilde{I}}}( \langle \nabla f( x^*), x_{T} - x^* \rangle) + \frac{l_s}{2} \| x_{T} - x^*\|^2 \nonumber \\
			& \leq f( x^*) +  \frac{ 1 }{2l_s} \|\pi_{\mathcal{\tilde{I}}}( \nabla f( x^*))\|^2 \nonumber\\
			& + \frac{l_s}{2  }\|x_{T} - x^* \|^2 + \frac{l_s}{2} \| x_{T} - x^*\|^2 \nonumber \\
			&=  f( x^*) +  \frac{ 1 }{2l_s} \|\pi_{\mathcal{\tilde{I}}}( \nabla f( x^*))\|^2 + l_s \| x_{T} - x^*\|^2. \nonumber
		\end{align}
		
		Take expectation on both sides,
		\begin{align*}
			E[f( x_{T}) - f( x^*) ] =  \frac{ 1 }{2l_s} \|\pi_{\mathcal{\tilde{I}}}( \nabla f( x^*))\|^2 + l_s E[\| x_{T} - x^*\|^2].
		\end{align*}
		
		From the upper bound of $E[\|x_{T} - x^*\|^2]$,
		\begin{align*}
			E[\|x_{T} - x^*\|^2] &\leq  \theta_2^{T} \|x_0 - x^*\|^2+  \frac{\xi_2}{1-\psi_2} \|\pi_{\mathcal{\tilde{I}}}( \nabla f(x^*)) \|^2. 
		\end{align*}
		
		Then we can get the final convergence result:
		\begin{align*}
			E[f( x_{T}) - f( x^*)]  & \leq \frac{ 1 }{2l_s} \|( \nabla f( x^*))\|^2 + l_s E[ \| x_{T} - x^*\|^2] \\
			&\leq  \theta_2^{T} l_s \|x_0 - x^*\|^2+  ( \frac{\xi_2}{1-\psi_2} + \frac{1}{2l_s}) \|\nabla f(x^*) \|^2 \\
			&=  \theta_2^{T}\Delta_2 + g_4( x^*)
		\end{align*}
		where $\Delta_2 = l_s \|x_0 - x^*\|^2$, $g_4( x^*)=  (\frac{\xi_2}{1-\psi_2}+ \frac{1}{2l_s}) \|\nabla f(x^*) \|^2  = O( \pi_{\mathcal{\tilde{I}}}( \|\nabla f(x^*) \|^2))$.
		
	\end{proof}
	
	\subsection{Proof of Corollary 4.1.3.}
	
	\begin{proof}
		Let { $Z = [Z^{(1)}; ...;Z^{(N)}]\in \mathbb{R}^{NB\times d}$} be the overall design matrix of the linear regression problem, and each row of $Z$ can be treated as drawn IID from a sub-Gaussian distribution with parameter $\sum_{i=1}^N \beta^{(i)}$. {$\epsilon = [\epsilon^{(1)};...; \epsilon^{(N)}]\in \mathbb{R}^{NB\times 1}$} is the random Gaussian noise. Then Lemma C.1 in \cite{wang2019differentially} immediately implies that $f_i$ is restricted $\rho_s$-strongly convex and  restricted $l_s$-strongly smooth with $\rho_s = \frac{4}{5} $ and $l_s = \frac{6}{5}$ respectively with probability at least $ (1 - \exp(-C_2 B)) $ if the total sample size $B \geq C_1 \tau  \log(d) \max_i\{(\beta^{(i)})^2\}$, where $C_1$ and $C_2 $ are universal constants. Furthermore, we know that $\|\nabla f(x^*)\|_{\infty} = \|\frac{Z^T\epsilon}{NB}\|_{\infty} \leq C_3 \sigma\sum_{i=1}^N \beta^{(i)}\sqrt{\frac{\log(d)}{NB}}$ 
		, with probability at least $(1-\exp(-C_4NB))$, 
		where $C_3, C_4$ are universal constants. Gathering everything  together yields the following bound with a high probability.
	\end{proof}
	
	\subsection{Proof of Corollary 4.1.4.}
	
	\begin{proof}
		If we further assume $\|z_{i,j}\|\leq \mathcal{K}$ and $C_{lower} \leq \exp(z_{i,j}^{T}x)/(1+\exp(z_{i,j}^{T}x))^2 \leq C_{upper}$ for $ i \in [N]$ and $j \in [B]$, the sparse logistic regression objective function is restricted $\rho_s$-strongly convex and  restricted $l_s$-strongly smooth with $\rho_s = \frac{4}{5} C_{lower} $ and $l_s = \frac{6}{5} C_{upper}$ respectively with a probability at least $ (1 - \exp(-C_6 B)) $ if $B \geq C_7 \tau \mathcal{K}^2 log(d)$, where $C_{lower}$, $C_{upper}$, $C_6$ and $C_7 $ are constants. Furthermore, according to Corollary 2 in \cite{loh2015regularized}, we have $\|\nabla f(x^*)\|_{\infty}  \leq C_8 \mathcal{K}\sqrt{\log(d)/NB}$ with a probability at least $(1-C_9
		exp(-C_{10}log( d))$, where $C_8$, $C_9$ and $C_{10}$ are universal constants. Therefore, we can obtain the following corollary. 
		Based on the above result, the estimation error specified in terms of the distance  $x_T$ and $x^*$ decreases when the total sample size $NB$ is large, or the dissimilarity level $\mathcal{B}$ and the dimension $d$ are small.
	\end{proof}

\end{document}